\documentclass[sigconf, table]{aamas}

\AtBeginDocument{%
  \providecommand\BibTeX{{%
    \normalfont B\kern-0.5em{\scshape i\kern-0.25em b}\kern-0.8em\TeX}}}

\usepackage{multirow}
\usepackage{array}
\usepackage{booktabs}
\usepackage{amsthm}
\usepackage{dsfont}
\usepackage{stmaryrd}
\usepackage{amsfonts}
\usepackage{mathtools}
\usepackage{wrapfig}
\usepackage[inline]{enumitem}
\usepackage{booktabs}
\usepackage{xcolor}
\usepackage{xargs}
\usepackage{enumitem}
\usepackage{nth}
\usepackage{flushend}
\setcopyright{ifaamas}  \copyrightyear{2020} \acmYear{2020} \acmDOI{} \acmPrice{} \acmISBN{} \acmConference[AAMAS'20]{Proc.\@ of the 19th International Conference on Autonomous Agents and Multiagent Systems (AAMAS 2020)}{May 9--13, 2020}{Auckland, New Zealand}{B.~An, N.~Yorke-Smith, A.~El~Fallah~Seghrouchni, G.~Sukthankar (eds.)}

\DeclareMathOperator*{\argmax}{argmax}
\usepackage{dsfont}

\begin{document}

\title{Safe Policy Improvement with an Estimated Baseline Policy}

\author{Thiago D. Sim\~{a}o}
\authornote{Work done while interning at Microsoft Research Montr\'{e}al.}
\affiliation{  \institution{Delft University of Technology}
  \city{The Netherlands}
    }
\email{t.diassimao@tudelft.nl}

\author{Romain Laroche}
\affiliation{  \institution{Microsoft Research Montr\'{e}al}
  \city{Canada}
    }
\email{romain.laroche@microsoft.com }

\author{ R\'{e}mi Tachet des Combes}
\affiliation{  \institution{Microsoft Research Montr\'{e}al}
  \city{Canada}
    }
\email{remi.tachet@microsoft.com }

\begin{abstract}
Previous work has shown the unreliability of existing algorithms in the batch Reinforcement Learning setting, and proposed the theoretically-grounded Safe Policy Improvement with Baseline Bootstrapping (SPIBB) fix: reproduce the baseline policy in the uncertain state-action pairs, in order to control the variance on the trained policy performance. However, in many real-world applications such as dialogue systems, pharmaceutical tests or crop management, data is collected under human supervision and the baseline remains unknown. In this paper, we apply SPIBB algorithms with a baseline estimate built from the data. We formally show safe policy improvement guarantees over the true baseline even without direct access to it. Our empirical experiments on finite and continuous states tasks support the theoretical findings. It shows little loss of performance in comparison with SPIBB when the baseline policy is given, and more importantly, drastically and significantly outperforms competing algorithms both in safe policy improvement, and in average performance.
\end{abstract}

\maketitle

\section{Introduction}
Reinforcement Learning (RL) is a framework for sequential decision-making optimization. Most RL research focuses on the online setting, where the system directly interacts with the environment and learns from it \cite{Mnih2015,vanSeijen2017}. While this setting might be the most efficient in simulation and in uni-device system control such as drones or complex industrial flow optimization, most real-world tasks (RWTs) involve a distributed architecture. We may cite a few: distributed devices (Internet of Things), mobile/computer applications (games, dialogue systems), or distributed lab experiments (pharmaceutical tests, crop management). These RWTs entail a high parallellization of the trajectory collection and strict communication constraints both in bandwidth and in privacy~\cite{Feraud2019}. Rather than spending a small amount of computational resource after each sample/trajectory collection, it is therefore more practical to collect a dataset using a behavioral (or baseline) policy, and then train a new policy from it. This setting is called \textit{batch RL}~\cite{Lange2012}.

Classically, batch RL algorithms apply dynamic programming on the samples in the dataset~\cite{Lagoudakis2003,Ernst2005}. \citet{Laroche2019} showed that in finite-state Markov Decision Processes (MDPs), these algorithms all converge to the same policy: the one that is optimal in the MDP with the maximum likelihood given the batch of data.
\citet{Petrik2016} show that this policy is approximately optimal to the order of the inverse square root of the minimal state-action pairs count in the dataset. Unfortunately, \citet{Laroche2019} show that even on very small tasks this minimal amount is almost always zero, and that, as a consequence, it gravely impairs the reliability of the approach: dynamic programming on the batch happens to return policies that perform terribly in the real environment. If a bad policy were to be run in distributed architectures such as the aforementioned ones, the consequences would be disastrous as it would jeopardize a high number of systems, or even lives.

Several attempts have been made to design reliable batch RL algorithms, starting with robust MDPs~\cite{Iyengar2005,Nilim2005}, which consist of considering the set of plausible MDPs given the dataset, and then find the policy for which the minimal performance over the robust MDPs set is maximal. The algorithm however tends to converge to policies that are unnecessarily conservative.

\citet{Xu2009} considered robust regret over the optimal policy: the algorithm searches for the policy that minimizes the maximal gap with respect to the optimal performance in every MDP in the robust MDPs. However, they proved that evaluating the robust optimal regret for a fixed policy is already NP-complete with respect to the state and action sets' size and the uncertainty constraints in the robust MDPs set.

Later, \citet{Petrik2016} considered the regret with respect to the behavioural policy performance over the robust MDPs set. The behavioural policy is called \textit{baseline} in this context. Similarly, they proved that simply evaluating the robust baseline regret is already NP-complete. Concurrently, they also proposed, without theoretical grounding, the Reward-adjusted MDP algorithm (RaMDP), where the immediate reward for each transition in the batch is penalized by the inverse square root of the number of samples in the dataset that have the same state and action than the considered transition.

Recently, \citet{Laroche2019} proposed Safe Policy Improvement with Baseline Bootstrapping (SPIBB), the first tractable algorithm with approximate policy improvement guarantees. Its principle consists in guaranteeing safe policy improvement by constraining the trained policy as follows: it has to reproduce the baseline policy in the uncertain state-action pairs. \citet{Nadjahi2019} further improved SPIBB's empirical performance by adopting soft constraints instead. Related to this track of research, \citet{Dias2019aaai,Dias2019ijcai} also developed SPIBB algorithms specifically for factored MDPs. Note that this thread of research is very distinct from online safe policy iteration, such as \cite{Kakade2002,Pirotta2013,Schulman2015,Schulman2017,Papini2017}, because the online setting allows them to perform very conservative updates.

Concurrently to robust approaches described above, another tractable and theoretically-grounded family of frequentist algorithms appeared under the name of High Confidence Policy Improvement~\cite[HCPI]{Paduraru2013,Mandel2014,thomas2015high}, relying on importance sampling estimates of the trained policy performance. The algorithm by \citet{Mandel2014}, based on concentration inequalities, tends to be conservative and requires hyper parameters optimization. The algorithms by \citet{thomas2015high2} rely on the assumption that the importance sampling estimate is normally distributed which is false when the number of trajectories is small. The algorithm by \citet{Paduraru2013} is based on bias corrected and accelerated bootstrap and tends to be too optimistic. In contrast with the robust approaches, from robust MDPs to Soft-SPIBB, HCPI may be readily applied to infinite MDPs with guarantees. However, it is well known that the importance sampling estimates have high variance, exponential with the horizon of the MDP. The SPIBB algorithm has a linear horizon dependency, given a fixed known maximal value and the common horizon/discount factor equivalence: $H=\frac{1}{1-\gamma}$~\cite{Kocsis2006}. Soft-SPIBB suffers a cubic upper bound but the empirical results rather indicate a linear dependency.

\citet{Nadjahi2019} perform a benchmark on randomly generated finite MDPs, baselines, and datasets. They report that the SPIBB and Soft-SPIBB algorithms are significantly the most reliable, and tie with RaMDP as the highest average performing algorithms. Additionally, they perform a benchmark on a continuous state space task, where the SPIBB and Soft-SPIBB algorithms significantly outperform RaMDP and Double-DQN~\cite{vanHasselt2016} both in reliability and average performance. Soft-SPIBB particularly shines in the continuous state experiments.

Despite these appealing results, there is a caveat: the SPIBB and Soft-SPIBB algorithms requires the baseline policy as input.
However, the behavior policy is not always available.
Consider for instance application involving human interactions, such as dialogue systems~\cite{Serban2016} and the medical sector.
In these situations it is common to have access to the observations and actions that were taken in a trajectory but not the policy that was followed.
To overcome this issue, \textit{we investigate the use of SPIBB and Soft-SPIBB algorithms in the setting where the baseline policy is unknown}.

Our aim is to answer a very natural question arising from the existing SPIBB analysis, whether access to the baseline is required or not.
Therefore, our contributions are threefold:
\begin{enumerate}
    \item We formally prove safety bounds for SPIBB and Soft-SPIBB algorithms with estimated baseline policies in finite MDPs (Section~\ref{sec:empiric}).
    \item We consolidate the theoretical results with empirical results in finite randomly generated MDPs, unknown baselines, and datasets (Section~\ref{sec:finite}, \url{https://github.com/RomainLaroche/SPIBB}).
    \item We apply the method on a continuous state task by investigating two types of behavioural cloning, and show that it outperforms competing algorithms by a large margin, in particular on small datasets (Section~\ref{sec:helico}, \url{https://github.com/rems75/SPIBB-DQN}).
\end{enumerate}
In summary, our results bring the SPIBB framework a step closer to many RWTs where the behavior policy is unknown.

\section{Background}
\label{sec:background}
This section reviews the previous technical results relevant for this work.

\subsection{Preliminaries}
A Markov Decision Process (MDP) is the standard formalism to model sequential decision making problems in stochastic environments.
An MDP $M$ is defined as $M = \langle \mathcal{X}, \mathcal{A}, P, R, \gamma \rangle$, where
    $\mathcal{X}$ is the state space,
    $\mathcal{A}$ is the set of actions the agent can execute,
    $P: \mathcal{X} \times \mathcal{A} \rightarrow \Delta_{\mathcal{X}}$ is the stochastic transition function,
    $R: \mathcal{X} \times \mathcal{A} \rightarrow [-R_{\max}, R_{\max}]$ is a stochastic immediate reward function,
    $\gamma$ is the discount factor. Without loss of generality, we assume that the initial state is deterministically $x_i$.

A policy $\pi : \mathcal{X} \rightarrow \Delta_{\mathcal{A}}$ represents how the agent interacts with the environment.
The value of a policy $\pi$ starting from a state $x \in \mathcal{X}$ is given by the expected sum of discounted future rewards:
\begin{align}
V^{\pi}_{M}(x)=\mathbb{E}_{\pi, M, x_0=x} \left[ \sum_{t \geq 0} \gamma^t R(x_t, a_t)\right].
\end{align}

Therefore, the performance of a policy, denoted $\rho(\pi, M)$, is the value in the initial state $x_i$.
The goal of a reinforcement learning agent is to find a policy $\pi: \mathcal{X} \rightarrow \Delta_{\mathcal{A}}$ that maximizes its expected sum of discounted rewards, however the agent does not have access to the dynamics of the true environment $M^* = \langle \mathcal{X}, \mathcal{A}, P^*, R^*, \gamma \rangle$.

In the batch RL setting, the algorithm receives as an input the dataset of previous transitions collected by executing a baseline policy $\pi_b$: $\mathcal{D} = \langle x_k,a_k, r_k, x'_k, t_k\rangle_{k\in\llbracket 1,|\mathcal{D}|\rrbracket}$, where the starting state of the transition is $x_k = x_i$ if $t_k = 0$ and $x_k = x'_{k-1}$ otherwise, $a_k \sim \pi_b(\cdot|x_k)$ is the performed action,  $r_{k} \sim R(x_k, a_k)$ is the immediate reward, $x'_{k} \sim P(\cdot|x_k, a_k)$ is the reached state, and the trajectory-wise timestep is $t_k=0$ if the previous transition was final and $t_k=t_{k-1}+1$ otherwise.

We build from a dataset $\mathcal{D}$ the Maximum Likelihood Estimate (MLE) MDP $\widehat{M} = \langle \mathcal{X}, \mathcal{A}, \widehat{P}, \widehat{R}, \gamma \rangle$, as follows:
\begin{align*}
    \widehat{P}(x'|x,a) &= \cfrac{N_\mathcal{D}(x,a,x')}{N_{\mathcal{D}}(x,a)}, \\
    \widehat{R}(x,a) &= \cfrac{\sum_{\langle x_j=x,a_j=a,r_j,x'_j \rangle\in\mathcal{D}} r_j}{N_{\mathcal{D}}(x,a)},
\end{align*}
where $N_{\mathcal{D}}(x,a) \text{ and } N_\mathcal{D}(x,a,x') $
are the state-action pair counts and next-state counts in the dataset $\mathcal{D}$. We also consider the robust MDPs set $\Xi$, \textit{i.e.} the set of plausible MDPs such that the true environment MDP $M^*$ belongs to it with high probability $1-\delta$:
\begin{align}
\begin{split}
  \Xi = &\left\{M = \langle \mathcal{X}, \mathcal{A}, R, P, \gamma\rangle
  \textnormal{ s.t. } \forall x,a, \right.\\
  &\left.\begin{array}{ll}
  ||P(\cdot|x,a)-\widehat{P}(\cdot|x,a)||_1 \leq e_\delta(x,a),\\
  |R(x,a)-\widehat{R}(x,a)| \leq e_\delta(x,a)R_{max}
  \end{array}\right\},
\end{split}
\end{align}
where $e_\delta(x,a)$ is a model error function on the estimates of $\widehat{M}$ for a state-action pair~$(x,a)$, which is classically upper bounded with concentration inequalities.

In the next section, we discuss an objective for these algorithms that aims to guarantee a safe policy improvement for the new policy.

\subsection{Approximate Safe Policy Improvement}
\citet{Laroche2019} investigate the setting where the agent receives as input the dataset $\mathcal{D}$ and must compute a new policy $\pi$ that approximately improves with high probability the baseline. Formally, the safety criterion can be defined as:
\begin{align}
\mathbb{P}\left(\rho(\pi, M^*) \geq \rho(\pi_b, M^*) - \zeta \right) \geq 1-\delta,
\end{align}
where $\zeta$ is a hyper-parameter indicating the improvement approximation and $1-\delta$ is the high confidence hyper-parameter. \citet{Petrik2016} demonstrate that the optimization of this objective is NP-hard. To make the problem tractable, \citet{Laroche2019} end up considering an approximate solution by maximizing the policy in the MLE-MDP while constraining the policy to be approximately improving in the robust MDPs set $\Xi$. More formally, they seek:
\begin{align*}
  \argmax_{\pi} \rho(\pi,\widehat{M}),  \textnormal{ s.t. } \forall M \in \Xi, \rho(\pi,M) \geq \rho(\pi_b,M) - \zeta.
\end{align*}

Given a hyper-parameter $N_\wedge$, their algorithm $\Pi_b$-SPIBB constrains the policy search to the set $\Pi_{b}$ of policies that reproduce the baseline probabilities in the state-action pairs that are present less than $N_\wedge$ times in the dataset $\mathcal{D}$:
\begin{align}
  \Pi_{b} = \left\{\pi \,\middle\vert\,\pi(a|x) = \pi_b(a|x)\;\textnormal{if}\;N_{\mathcal{D}}(x,a)<N_\wedge\right\}.
\end{align}

We now recall the safe policy improvement guaranteed by the algorithm $\Pi_b$-SPIBB:\begin{theorem}[Safe policy improvement with baseline bootstrapping]
  Let $\pi_b^*$ be the optimal policy constrained to $\Pi_b$ in the MLE-MDP. Then, $\pi_b^*$ is a $\zeta$-approximate safe policy improvement over the baseline $\pi_b$ with high probability $1-\delta$, where:
  \begin{equation*}
  \zeta = \cfrac{4 V_{max}}{1-\gamma} \sqrt{\cfrac{2}{N_{\wedge}}\log\cfrac{2|\mathcal{X}||\mathcal{A}|2^{|\mathcal{X}|}}{\delta}}  - \rho(\pi^*_b, \widehat{M}) + \rho(\pi_b, \widehat{M}).
  \end{equation*}
  \label{th:safepolicyimprovement-pi}
\end{theorem}

Our work also considers the algorithm Soft-SPIBB \citep{Nadjahi2019}, that constrains the policy search such that the cumulative state-local error never exceeds $\epsilon$, with $\epsilon$ a fixed hyper-parameter. More formally, the policy constraint is expressed as follows:
\begin{align}
  \Pi_{\sim} = \left\{\pi \,\middle\vert\, \forall x,
            \sum_{a \in \mathcal{A}} e_\delta(x,a) \big\lvert\pi(a|x)-\pi_b(a|x)\big\rvert \leq \epsilon\right\}.
\end{align}

Under some assumptions, \citet{Nadjahi2019} demonstrate a looser safe policy improvement bound. Nevertheless, the policy search is less constrained and their empirical evaluation reveals that Soft-SPIBB safely finds better policies than SPIBB.

Both algorithms presented in this section assume the behavior policy $\pi_b$ is known and can be used during the computation of a new policy.
In the next section, we get to the main contribution of this paper, where we investigate how these algorithms can be applied when $\pi_b$ is not given.

\section{Baseline Estimates}
\label{sec:empiric}
In this section, we consider that the true baseline is unknown and implement a baseline estimate in order for the SPIBB and Soft-SPIBB algorithms to still be applicable.
Before we start our analysis, we present an auxiliary lemma.

Let $d_{M}^\pi(x,a)$ be the discounted sum of visits of state-action pair $(x,a) \in \mathcal{X} \times \mathcal{A}$ while following policy $\pi$ in MDP $M$ and $d_\mathcal{D}$ is the state-action discounted distribution in dataset $\mathcal{D}$.

\begin{lemma} \label{lem:dist}
Considering that the trajectories in $\mathcal{D}$ are i.i.d. sampled, the $\text{L}_1$ deviation of the empirical discounted sum of visits of state-action pairs is bounded.
We have the following concentration bound:
\begin{align}
    \mathbb{P}\left(\big\lVert d_{M^*}^{\pi_b}-d_\mathcal{D}\big\rVert_1 (1-\gamma) \geq \varepsilon\right)
    \leq \left(2^{|\mathcal{X}||\mathcal{A}|}-2\right)\exp\left(-\cfrac{N\varepsilon^2}{2}\right),
\end{align}
where $N$ is the number of trajectories in $\mathcal{D}$.
\end{lemma}
\begin{proof}
Let $\mathcal{T} = (\mathcal{X} \times \mathcal{A})^\mathbb{N}$ denote the set of trajectories and $T = (T_1, \dots, T_N)$ be a set of N $\mathcal{T}$-valued random variables. For a given $E \subset \mathcal{X} \times \mathcal{A}$, we define the function $f_E$ on $\mathcal{T}$ as:
\begin{align*}
    f_E(T) = f_E(T_1, \dots, T_N) \coloneqq (1 - \gamma) \sum_{i = 1}^N \sum_{t \geq 0} \gamma^t \mathds{1}(T_i^t \in E),
\end{align*}
where $T_i^t$ is the state-action pair on trajectory $i$ at time $t$.
In particular, we have that
\begin{align}\label{eq:ocu_hat}
    f_E(\mathcal{D}) = N (1-\gamma) d_\mathcal{D}(E)  \text{ and}
\end{align}
\begin{align}\label{eq:ocu_exp}
    \mathds{E}[f_E(T)] = N (1 - \gamma) d_{M^*}^{\pi_b}(E) ,
\end{align}
where $d_\mathcal{D}(E)$ and $d_{M^*}^{\pi_b}(E)$ denote the mass of set $E$ under $d_\mathcal{D}$ and $d_{M^*}^{\pi_b}$ respectively.

For two sets $T$ and $T'$ differing only on one trajectory, say the $k$-th, we have:
\begin{align*}
    |f_E(T) - f_E(T')| = |(1 - \gamma) \sum_{t \geq 0} \gamma^t \left(\mathds{1}(T_k^t \in E) - \mathds{1}({T'}_k^t \in E)\right)| \leq 1.
\end{align*}
This allows us to apply the independent bounded difference inequality by \citet[Theorem 3.1]{mcdiarmid1998concentration}, which gives us:
\begin{align}\label{eq:bounded_diff}
    \mathbb{P}\left( f_E(T) - \mathds{E}[f_E(T)] \geq  \bar{\varepsilon}\right)
    \leq \exp\left(-2\cfrac{\bar{\varepsilon}^2}{N}\right).
\end{align}
We know that
\[
    \big\lVert d_{M^*}^{\pi_b}-d_\mathcal{D}\big\rVert_1 (1-\gamma) = \displaystyle{\max_{E \subset \mathcal{X} \times \mathcal{A}}} 2 (1 - \gamma) (d_\mathcal{D}(E) - d_{M^*}^{\pi_b}(E)).
\]
This guarantees from a coarse union bound and equations \ref{eq:ocu_hat}, \ref{eq:ocu_exp} and \ref{eq:bounded_diff} that:
\begin{align*}
    \mathbb{P} &\left(\big\lVert d_{M^*}^{\pi_b}-d_\mathcal{D}\big\rVert_1 (1-\gamma) \geq \varepsilon\right) \\
        & \leq \sum_{E \subset \mathcal{X} \times \mathcal{A}} \mathbb{P} \left( (1 - \gamma) (d_\mathcal{D}(E) - d_{M^*}^{\pi_b}(E)) \geq \frac{\varepsilon}{2} \right) \\
        &= \sum_{E \subset \mathcal{X} \times \mathcal{A}} \mathbb{P} \left( (1 - \gamma) \left(\frac{f_E(\mathcal{D})}{N(1-\gamma)} - \frac{\mathbb{E}[f_E(\mathcal{D})]}{N(1-\gamma)}\right) \geq \frac{\varepsilon}{2} \right) \\
        &\leq \sum_{E \subset \mathcal{X} \times \mathcal{A}} \exp\left(-2\cfrac{{(\frac{N\varepsilon}{2})}^2}{N}\right)    \\
        & \leq \left(2^{|\mathcal{X}||\mathcal{A}|}-2\right) \exp\left(-\frac{N\varepsilon^2}{2}\right),
\end{align*}
where in the sum over subsets, we ignored the empty and full sets for which the probability is trivially 0.
\end{proof}

\subsection{Algorithm and analysis}
We construct the Maximum Likelihood Estimate of the baseline $\widehat{\pi}_b$ (MLE baseline) as follows:
\begin{align}
\widehat{\pi}_b(a|x) = \begin{cases}
        \frac{N_\mathcal{D}(x,a)}{N_\mathcal{D}(x)} & \text{if } N_\mathcal{D}(x)>0,\\
        \frac{1}{|\mathcal{A}|}      & \text{otherwise},
    \end{cases}
    \label{eq:pi_hat}
\end{align}
where $N_\mathcal{D}(x)$ is the number of transitions starting from state $x$ in dataset $\mathcal{D}$. Using this MLE policy, we may prove approximate safe policy improvement:

\begin{theorem}[Safe policy improvement with a baseline estimate]
    Given an algorithm $\alpha$ relying on the baseline $\pi_b$ to train a $\zeta$-approximate safe policy improvement $\pi^*_b$ over $\pi_b$ with high probability $1-\delta$. Then, $\alpha$ with an MLE baseline $\widehat{\pi}_b$ allows to train a $\widehat{\zeta}$-approximate safe policy improvement $\widehat{\pi}^*_b$ over $\pi_b$ with high probability $1-\widehat{\delta}$:
    \begin{align}
        \widehat{\delta} &= \delta + 2\delta', \\
        \widehat{\zeta} &= \zeta + \cfrac{2 R_{\max}}{1-\gamma}\sqrt{\cfrac{3|\mathcal{X}||\mathcal{A}| + 4\log\frac{1}{\delta'}}{2N}},
            \end{align}
    where $N$ is the number of trajectories in the dataset $\mathcal{D}$ and $1-\delta'$ controls the uncertainty stemming from the baseline estimation.
    \label{thm:empiric}
\end{theorem}
\begin{proof}
We are ultimately interested in the performance improvement of $\widehat{\pi}^*_b$ with respect to the true baseline $\pi_b$ in the true environment $M^*$. To do so, we decompose the difference into two parts:
\begin{align}
\begin{split}
     \rho(\widehat{\pi}^*_b, M^*)-\rho(\pi_b, M^*) &= \underbrace{\rho(\widehat{\pi}^*_b, M^*) - \rho(\widehat{\pi}_b, M^*)}_\text{$\alpha$-SPI guarantee} \\
     & + \underbrace{\rho(\widehat{\pi}_b, M^*) -\rho(\pi_b, M^*)}_\text{baseline estimate approximation}. \label{eq:newdev}
\end{split}
\end{align}

Regarding the first term, note that, while $\widehat{\pi}_b$ is not the true baseline, it is the MLE baseline, meaning in particular that it was more likely to generate the dataset $\mathcal{D}$ than the true one. Hence, we may consider it as a potential behavioural policy, and apply the safe policy improvement guarantee provided by algorithm $\alpha$ to bound the difference.

Regarding the second term, we need to use the distributional formulation of the performance of any policy $\pi$:
\begin{align}
     \rho(\pi, M) &= \sum_{x\in\mathcal{X}} \sum_{a\in\mathcal{A}} d_{M}^\pi(x,a)\mathbb{E}[R(x,a)].
\end{align}
Then, we may rewrite the second term in Equation~\ref{eq:newdev} and upper bound it using H\"older's inequality as follows:
\begin{align}
\begin{split}
     \sum_{x\in\mathcal{X}} \sum_{a\in\mathcal{A}} & \left(d_{M^*}^{\widehat{\pi}_b}(x,a)-d_{M^*}^{\pi_b}(x,a)\right) \mathbb{E}[R^*(x,a)] \\
     &\leq \big\lVert d_{M^*}^{\widehat{\pi}_b}-d_{M^*}^{\pi_b}\big\rVert_1 R_{\max}.
\end{split}
\end{align}

Next, we decompose the state-action discounted visits divergence as follows:
\begin{align}
\begin{split}
     \big\lVert d_{M^*}^{\widehat{\pi}_b}-d_{M^*}^{\pi_b}\big\rVert_1 \leq
     \underbrace{\big\lVert d_{M^*}^{\pi_b}-d_\mathcal{D}\big\rVert_1}_\text{Lemma~\ref{lem:dist}}
     + \underbrace{\big\lVert d_{M^*}^{\widehat{\pi}_b}-d_\mathcal{D}\big\rVert_1.}_\text{positive correlation} \label{eq:correlation}
 \end{split}
\end{align}
For the first term, we can use the concentration inequality from Lemma~\ref{lem:dist}\footnote{We need to rescale with $(1-\gamma)$ the state-action discounted visits to make it sum to 1 since the original bound applies to probability distributions.}.
With a little calculus and by setting the right value to~$\varepsilon$, we obtain with high probability $1-\delta'$:
\begin{align*}
     \big\lVert d_{M^*}^{\pi_b}-d_\mathcal{D}\big\rVert_1 \leq \cfrac{1}{1-\gamma}\sqrt{\cfrac{3|\mathcal{X}||\mathcal{A}| + 4\log\frac{1}{\delta'}}{2N}}.
    \end{align*}

Regarding the second term of Equation~\ref{eq:correlation}, we may observe that there is a correlation between $\widehat{\pi}_b$ and $d_\mathcal{D}$ through $\mathcal{D}$, but it is a positive correlation, meaning that the divergence between the distributions is smaller than the one with an independently drawn dataset of the same size. As a consequence, we are also able to upper bound it by assuming independence, and using the same development as for the first term. This finally gives us from Equation~\ref{eq:correlation} and with high probability $1-2\delta'$:
\begin{align}
     \big\lVert d_{M^*}^{\widehat{\pi}_b}-d_{M^*}^{\pi_b}\big\rVert_1 &\leq \cfrac{2}{1-\gamma}\sqrt{\cfrac{3|\mathcal{X}||\mathcal{A}| + 4\log\frac{1}{\delta'}}{2N}},
    \end{align}
which allows us to conclude the proof using union bounds.
\end{proof}

\subsection{Theorem~\ref{thm:empiric} discussion}
SPIBB and Soft-SPIBB safe policy improvement guarantees exhibit a trade-off (controlled with their respective hyper-parameters $\frac{1}{\sqrt{N_\wedge}}$ and $\epsilon$) between upper bounding the true policy improvement error (first term in Theorem~\ref{th:safepolicyimprovement-pi}) and allowing maximal policy improvement in the MLE MDP (next terms). When the hyper-parameters are set to 0, the true policy improvement error is null, because, trivially, no policy improvement is allowed: the algorithm is forced to reproduce the baseline. When the hyper-parameters grow, larger improvements are permitted, but the error upper bound term also grows. When the hyper-parameters tend to $+\infty$, the algorithms are not constrained anymore and find the optimal policy in the MLE MDP. In that case, the error is no longer upper bounded, resulting in poor safety performance.

When using the MLE baseline instead of the true baseline, Theorem~\ref{thm:empiric} introduces another error upper bound term accounting for the accurateness of the baseline estimate that cannot be reduced by hyper-parameter settings. That fact is entirely expected, as otherwise we could consider an empty dataset, pretend it was generated with an optimal policy and expect a safe policy improvement over it. Another interesting point is that the bound depends on the number of trajectories, not the number of state-action visits, nor the total number of samples. Indeed, even with a huge number of samples, if there were collected only from a few trajectories, the variance may still be high, since future states visited on the trajectory depend on the previous transitions.

\begin{figure*}[t]
\centering
\hfill
\includegraphics[width=0.47\textwidth]{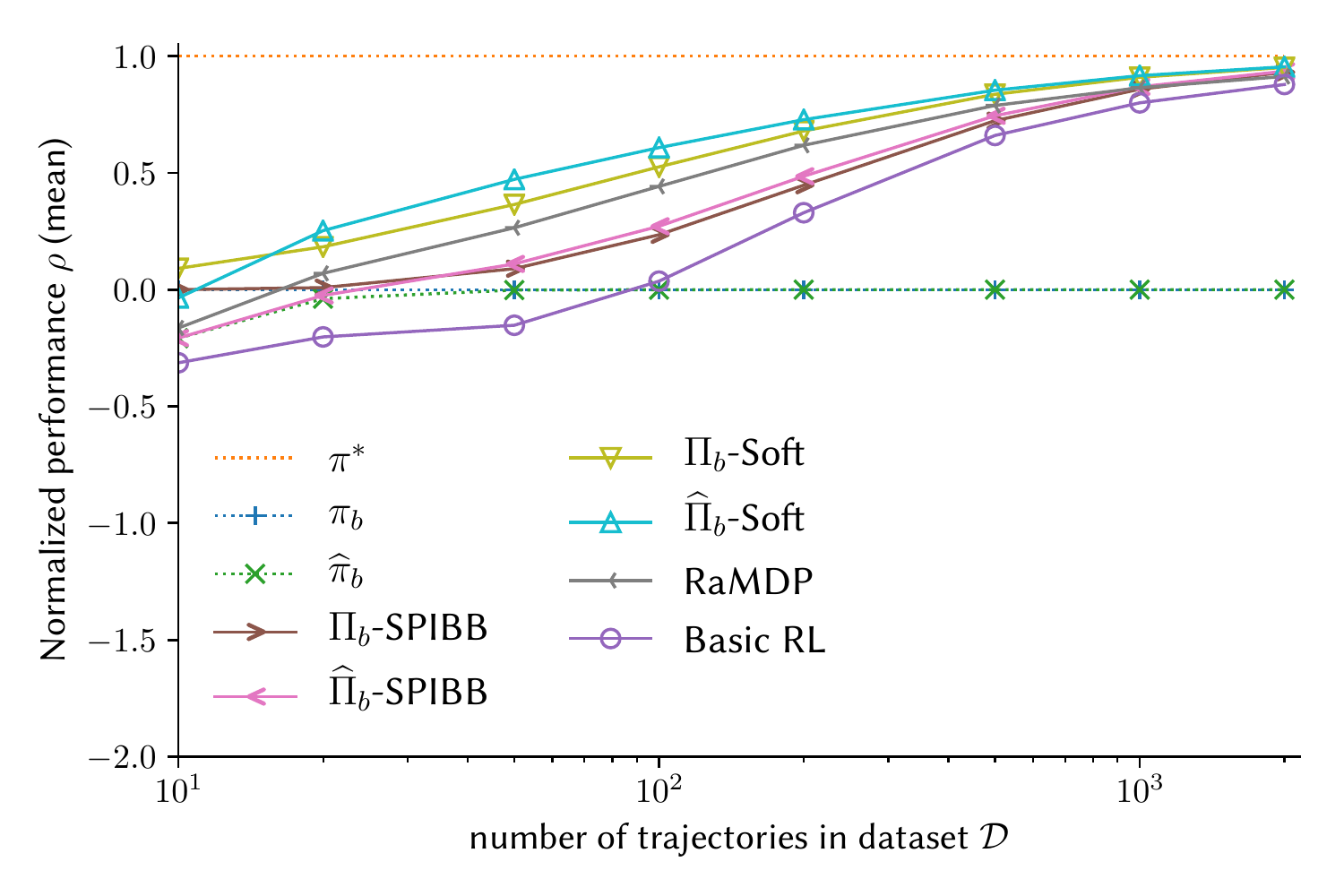}\hfill
\hfill
\includegraphics[width=0.47\textwidth]{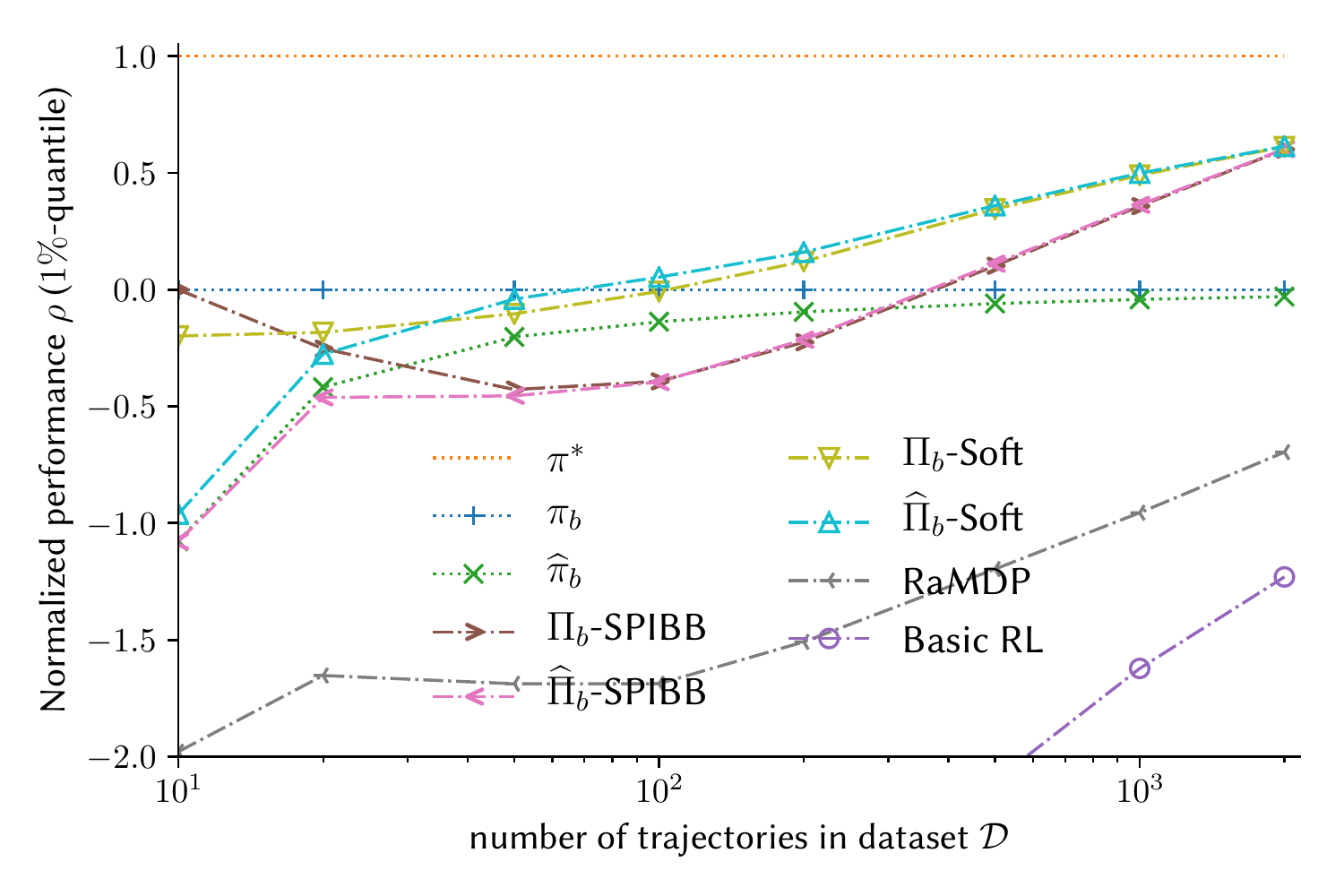}
\hfill
\caption{Finite MDPs with \boldmath$\eta = 0.9$, $N_\wedge=7$ and $\epsilon=0.5$. On the left, the mean curves, on the right, the 1\%-quantile curves.}
\label{fig:fin_1}
\end{figure*}

Regarding the MDP parameters dependency, the upper bound grows as the square root of the state set size, as for standard SPIBB, but also grows as the square root of the action set size contrarily to SPIBB that has a logarithmic dependency, which may cause issues in some RL problems. The direct horizon dependency is the same (linear). But one could argue that it is actually lower. The maximal value $V_{max}$ in the SPIBB bounds can reach $\frac{R_{max}}{1-\gamma}$, making the dependency in $H$ quadratic, while the $N$ in our denominator may be regarded as a hidden horizon (since $N\approx \frac{|\mathcal{D}|}{H}$), making the total dependency $\approx H^{3/2}$. In both cases, those are better than the Soft-SPIBB cubic dependency.

One may consider other baseline estimates than the MLE, using Bayesian priors for instance, and infer new bounds. This should work as long as the baseline estimate remains a policy that could have generated the dataset.

\section{Empirical Analysis}
Our empirical analysis reproduces the most challenging experiments found in \citet{Laroche2019} and \citet{Nadjahi2019}.
We split it in two parts, the first considers random MDPs with finite state spaces and the second MDPs with continuous state spaces.

\subsection{Random finite MDPs}
\label{sec:finite}
\subsubsection{Setup:}
The objective of this experiment is to empirically analyse the consistency between the theoretical findings and the practice.
The experiment is run on finite MDPs that are randomly generated, with randomly generated baseline policies from which trajectories are obtained.
We recall the setting below.

The true environment is a randomly generated MDP with 50 states, 4 actions, and a transition connectivity of 4: a given state-action pair may transit to 4 different states at most.
The reward function is 0 everywhere, except for transitions entering the goal state, in which case the trajectory terminates with a reward of 1.
The goal state is the hardest state to reach from the initial one.

The baselines are also randomly generated with a predefined level of performance specified by a ratio $\eta$ between the optimal policy $\pi^*$ performance and the uniform policy $\tilde{\pi}$ performance:
$
    \rho(\pi_b,M) = \eta \rho(\pi^*,M) + (1-\eta) \rho(\tilde{\pi},M).
$
For more details on the process, we refer the interested reader to the original papers.
Two values for $\eta$ were considered: the experiments with $\eta=0.9$ are reported here.
The experiments with $\eta=0.1$ had similar results and are omitted for lack of space.
We also study the influence of the dataset size $|\mathcal{D}| \in [10, 20, 50, 100, 200, 500, 1000, 2000]$.

\subsubsection{Competing algorithms:}
Our plots display nine curves:
\begin{itemize}
    \item $\pi^*$: the optimal policy,
    \item $\pi_b$: the true baseline,
    \item $\widehat{\pi}_b$: the MLE baseline,
    \item $\Pi_b$/$\widehat{\Pi}_b$-SPIBB: SPIBB with their respective baselines,
    \item $\Pi_b$/$\widehat{\Pi}_b$-Soft: Soft-SPIBB with their respective baselines,
    \item RaMDP: Reward-adjusted MDP,
    \item and Basic RL: dynamic programming on the MLE MDP.
\end{itemize}

All the algorithms are compared using their optimal hyper-parameter according to previous work.
Our hyper-parameter search with the MLE baselines did not show significant differences and we opted to report results with the same hyper-parameter values.
Soft-SPIBB algorithms are the ones coined as Approx. Soft SPIBB by \citet{Nadjahi2019}.

\subsubsection{Performance indicators:}
Given the random nature of the MDP and baseline generations, we normalize the performance to allow inter-experiment comparison:
\begin{align}
    \rho=\cfrac{\rho(\pi,M^*) - \rho(\pi_b,M^*)}{\rho(\pi^*,M^*)-\rho(\pi_b,M^*)}.
\end{align}
Thus, the optimal policy always has a normalized performance of 1, and the true baseline a normalized performance of 0.
A positive normalized performance means a policy improvement, and a negative normalized performance an infringement of the policy improvement objective.
Figures either report the average normalized performance of the algorithms or its $1\%$-quantile\footnote{Note the difference with previously reported results in SPIBB papers, which focused on the conditional value at risk indicator.}.
Each setting is processed on 250k seeds, to ensure that every performance gap visible to the naked eye is significant.

\subsubsection{Empirical results:}
Figure~\ref{fig:fin_1} shows the results with $\eta=0.9$, \textit{i.e.} the hard setting where the behavior baseline is almost optimal, and therefore difficult to improve.

\textit{Performance of the MLE baseline.}
First, we notice that the mean performance of the MLE baseline $\widehat{\pi}_{b}$ is slightly lower than the true baseline policy $\pi_b$ for small datasets.
As $|\mathcal{D}|$ increases, the performance of  $\widehat{\pi}_{b}$ quickly increases to reach the same level.
The $1\%$-quantile is significantly lower when the number of trajectories is reduced.

\textit{Soft-SPIBB with true and estimated baselines.}
Comparing the results of $\Pi_b$-Soft and $\widehat{\Pi}_{b}$-Soft curves, it is surprising that the policy computed using an estimated policy as a baseline yields better results than the one computed with the true policy.
Notice that the estimated baseline $\widehat{\pi}_{b}$ has a higher variance than the true baseline~$\pi_{b}$.
If we consider the impact of this variance in a given state, it means that sometimes the best (resp. worst) action will be taken more often (resp. less).
When it is the case, the trained policy will be better than what could have been done with the true baseline.
Sometimes, the opposite will happen, but in this case, the algorithm will try to avoid reaching this state and choose an alternative path.
This means that in expectation, this does not average out and the variance in the baseline estimation might be beneficial.

\textit{SPIBB with true and estimated baselines.}
Analysing the performance of the $\widehat{\Pi}_{b}$-SPIBB algorithm, we notice that it also slightly improves over $\Pi_b$-SPIBB on the mean normalized performance.
As far as safety is concerned, we see that the $1\%$-quantile of policies computed with $\widehat{\Pi}_{b}$-SPIBB falls close to the $1\%$-quantile of the estimated baseline $\widehat{\pi}_{b}$ for small datasets and close to the $1\%$-quantile of the policies $\Pi_b$-SPIBB for datasets with around 100 trajectories. It is expected as $\widehat{\Pi}_{b}$-SPIBB tends to reproduce the baseline for very small datasets, and improves over it for larger ones. That statement is also true of $\widehat{\Pi}_{b}$-Soft.

\textit{RaMDP and Basic RL.}
Finally, it is interesting to observe that although RaMDP and Basic RL can compute policies with rather high mean performance, these algorithms often return policies performing much worse than the MLE policy $\widehat{\pi}_{b}$ (as seen in their $1\%$-quantile).

\begin{wrapfigure}{r}{0.35\columnwidth}
\centering
\includegraphics[width=0.35\columnwidth]{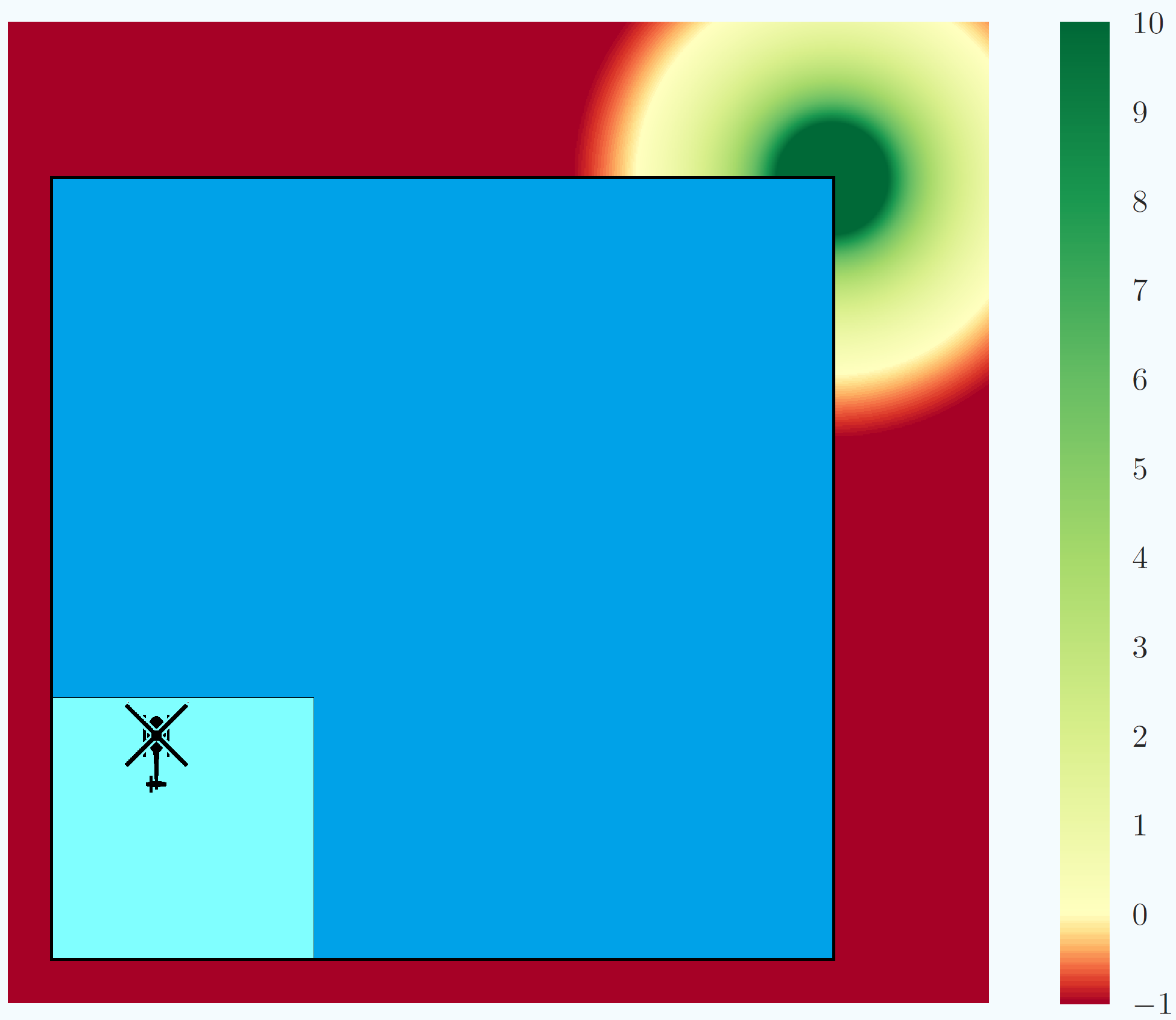}
\caption{Helicopter.}
\label{fig:helico}
\end{wrapfigure}
\subsection{Continuous MDPs}
\label{sec:helico}
\subsubsection{Helicopter domain:}
For MDPs with continuous state space, we focus on the helicopter environment \citep[Figure~\ref{fig:helico}]{Laroche2019}.
In this stochastic domain, the state is defined by the position and velocity of the helicopter.
The agent has a discrete set of 9 actions to control the thrust applied in each dimension.
The helicopter begins in a random position of the bottom-left corner with a random initial velocity.
The episode ends if the helicopter's speed exceeds some threshold, giving a reward of -1, or if it leaves the valid region, in which case the agent gets a reward between -1 and 10 depending on how close it is to the top-right corner. Using a fixed behavior policy $\pi_b$ we generate $1,000$ datasets for each algorithm. We report results for two dataset sizes: $3,000$ and $10,000$ transitions.

\begin{figure*}[t!]
\centering
\includegraphics[width=0.99\textwidth]{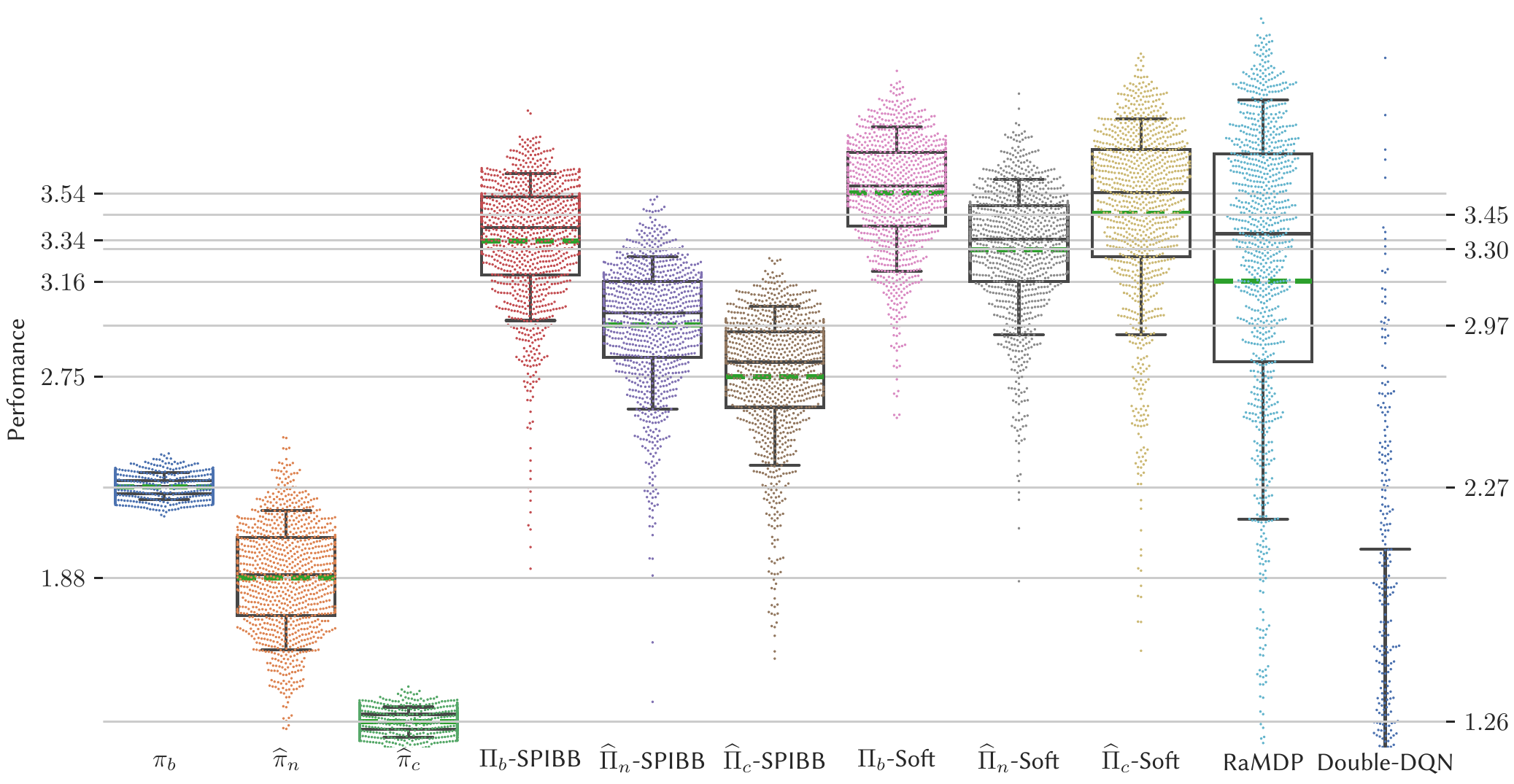}
\caption{\boldmath$|\mathcal{D}| = 10,000$. The green dashed line shows the average and the caps show the 10\% and 90\% percentile. Each dot on the swarm plots displays the evaluation of a seed.}
\label{fig:cont_large}
\end{figure*}

\begin{figure*}[t!]
\centering
\includegraphics[width=0.99\textwidth]{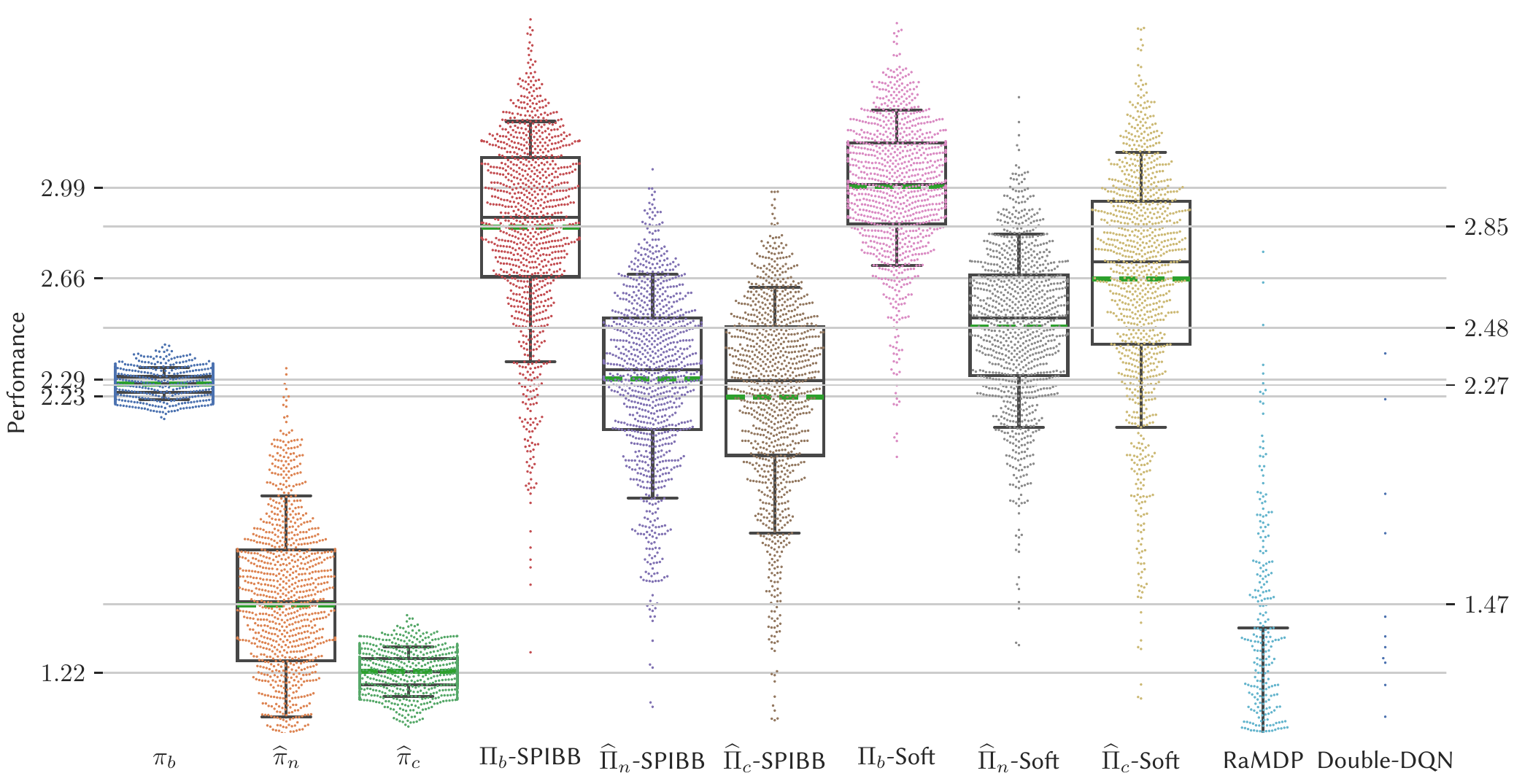}
\caption{\boldmath$|\mathcal{D}| = 3,000$. The green dashed line shows the average and the caps show the 10\% and 90\% percentile. Each dot on the swarm plots displays the evaluation of a seed.}
\label{fig:cont_small}
\end{figure*}

\subsubsection{Behavioural cloning:}
In infinite MDPs, there is no MLE baseline definition. We have to lean on behavioural cloning techniques. We compare here two straightforward ones in addition to the true behavior policy $\pi_b$: a baseline estimate $\widehat{\pi}_c$ based on the same pseudo-counts used by the algorithms, and a neural-based baseline estimate $\widehat{\pi}_n$ that uses a standard probabilistic classifier.

The \textit{count-based policy} follows a principle similar to the MLE policy.
It uses a pseudo-count for state-action pairs $\tilde{N}(x, a)$ defined according to the sum of the euclidean distance $\lVert x-x'\rVert_2$ from the state $x$ and all states of transitions in the dataset where the action~$a$ was executed \cite[Section 3.4]{Laroche2019}:
\begin{align}
    \tilde{N}_{\mathcal{D}}(x,a) = \sum_{\langle x_j,a_j=a,r_j,x'_j \rangle\in\mathcal{D}} \max\left\{0, 1- \frac{\lVert x-x_j\rVert_2}{d_0}\right\},
\end{align}
where $d_0$ is a hyper-parameter to impose a minimum similarity before increasing the counter of a certain state.
We also compute the state pseudo-count using this principle: $\tilde{N}_{\mathcal{D}}(x) = \sum_{a\in \mathcal{A}}\tilde{N}_{\mathcal{D}}(x,a)$.
This way, we can define the count-based baseline estimate replacing the count in Equation~\ref{eq:pi_hat} by its pseudo-count counterpart:
\begin{align}
\widehat{\pi}_c(a|x) = \begin{cases}
        \frac{\tilde{N}_\mathcal{D}(x,a)}{\tilde{N}_\mathcal{D}(x)} & \text{if } \tilde{N}_\mathcal{D}(x)>0,\\
        \frac{1}{|\mathcal{A}|}      & \text{otherwise}.
    \end{cases}
    \label{eq:pseudo_count_estimate}
\end{align}

The \textit{neural-based policy}  $\widehat{\pi}_n(a|x)$ is estimated using a supervised learning approach.
We train a probabilistic classifier using a neural network to minimize the negative log-likelihood with respect to the actions in the dataset.

We use the same architecture as the one used to train the Double-DQN models, which is shared among all the algorithms in the helicopter domain experiments: a fully connected neural network with 3 hidden layers of 32, 128 and 28 neurons respectively, and 9 outputs corresponding to the 9 actions.

To avoid overfitting, we split the dataset in two parts: 80\% for training and 20\% for validation. During training, we evaluate the classifier on the validation dataset at the end of every epoch and keep the network with the smallest validation loss.

\begin{table*}[t!]
    \def\arraystretch{1.2}
    \centering
    \begin{tabular}{|cc|cccc|cccc|}
    \hline
            &  &  \multicolumn{4}{c|}{$|\mathcal{D}| = 3,000$} & \multicolumn{4}{c|}{$|\mathcal{D}| = 10,000$}  \\
           Baseline & Algorithm & $\mathbb{P}\left(\rho(\pi) > \rho(\pi_b)\right)$ & avg perf & 10\%-qtl & 1\%-qtl & $\mathbb{P}\left(\rho(\pi) > \rho(\pi_b)\right)$ & avg perf & 10\%-qtl & 1\%-qtl \\
           \hline
                 $\pi_b$   &            &       0.499    & 2.27 &  2.22 &  2.18    & 0.499 &  2.27 &  2.22 &  2.18 \\
       $\widehat{\pi}_{n}$ & baseline   &       0.002    & 1.47 &  1.06 &  0.75    & 0.032 &  1.88 &  1.57 &  1.34 \\
       $\widehat{\pi}_{c}$ &            &       0.000    &  1.22 &  1.13 &  1.05    & 0.000 &  1.26 &  1.19 &  1.14  \\
          \hline
             $\pi_b$       &            &       0.928    &  2.85 &  2.36 &  1.90    & 0.992 &  3.34 &  2.99 &  2.39 \\
 $\widehat{\pi}_{n}$       & SPIBB      &       0.582    &  2.29 &  1.86 &  1.43    & 0.973 &  2.97 &  2.61 &  2.15  \\
 $\widehat{\pi}_{c}$       &            &       0.514    &  2.23 &  1.73 &  1.21    & 0.930 &  2.75 &  2.37 &  1.75 \\
          \hline
              $\pi_b$      &            &       0.990 &  2.99 &  2.71 &  2.31    & 1.000 &  3.54 &  3.21 &  2.82 \\
  $\widehat{\pi}_{n}$      & Soft-SPIBB &       0.760    &  2.48 &  \textbf{2.12} &  \textbf{1.71}    & \textbf{0.996} &  3.30 &  \textbf{2.93} &  \textbf{2.47}  \\
  $\widehat{\pi}_{c}$      &            & \textbf{0.785} &  \textbf{2.66} &  \textbf{2.11} &  1.51    & 0.980 &  \textbf{3.45} &  \textbf{2.93} &  2.09 \\
          \hline
                     N/A   & RaMDP      &       0.006    &  0.37 & -0.75 & -0.99    & 0.876  & 3.16 &  2.13 &  0.23  \\
                     N/A   & Double-DQN        &       0.001    & -0.77 & -1.00 & -1.00    & 0.076  & 0.25 & -0.97 & -1.00 \\
    \hline
    \end{tabular}
    \caption{Numerical results for the two size of datasets. The key performance indicators are respectively the percentage of policy improvement over the true baseline, the average performance of the trained policies, the 10\%-quantile, and the 1\%-quantile. For each column, we bold the best performing algorithm that is not using the true baseline \boldmath$\pi_b$.}
    \label{tab:percentage_above_baseline_mean}
\end{table*}

\subsubsection{Competing algorithms:}
\begin{itemize}
    \item $\pi_b$: the true baseline,
    \item $\widehat{\pi}_c$: the pseudo-count-based estimate of the baseline,
    \item $\widehat{\pi}_n$: the neural-based estimate of the baseline,
    \item $\Pi_b$/$\widehat{\Pi}_c$/$\widehat{\Pi}_n$-SPIBB: SPIBB with their respective baselines,
    \item $\Pi_b$/$\widehat{\Pi}_c$/$\widehat{\Pi}_n$-Soft: Soft-SPIBB with their respective baselines,
    \item RaMDP: Double-DQN with Reward-adjusted MDP,
    \item and Double-DQN: basic deep RL algorithm.
\end{itemize}

\subsubsection{Hyper-parameters}
Building on the results presented by \citet{Nadjahi2019}, we set the hyper-parameters for the experiments with $|\mathcal{D}| = 10,000$ ($|\mathcal{D}| = 3,000$) as follows: $\Pi_{b}$-SPIBB with $N_\wedge=3$ ($N_\wedge=1$), $\Pi_{b}$-Soft with $\epsilon=0.6$ ($\epsilon=0.8$), and RaMPD with $\kappa=1$ ($\kappa=1.75$). For the algorithms using an estimated baseline we run a parameter search considering $N_\wedge \in [2, 3, 4, 5]$ ($N_\wedge \in [0.5, 1, 2, 3]$) for SPIBB and $\epsilon \in [0.4, 0.6, 0.8, 1]$ ($\epsilon \in [0.6, 0.8, 1, 1.2, 1.5, 1.8, 2]$) for {Soft-SPIBB} and set the parameters for the main experiments as follows: $\widehat{\Pi}_{n}$-SPIBB  and  $\widehat{\Pi}_{c}$-SPIBB  with $N_\wedge=3.0$ ($N_\wedge=1.0$), and $\widehat{\Pi}_{n}$-Soft and  $\widehat{\Pi}_{c}$-Soft with $\epsilon=0.6$ ($\epsilon=0.8$).

\subsubsection{Performance indicators:}
The plots represent for each algorithm a modified box-plot where the caps show the $10\%$-quantile and $90\%$-quantile, the upper and lower limits of the box are the $25\%$ and $75\%$ quantiles and the middle line in black shows the median. We also show the average of each algorithm (dashed lines in green) and finally add a swarm-plot to enhance the distribution visualization.
The table provides additional details, including the percentage of policies that showed a performance above the average performance of the true baseline policy.

\subsubsection{Results:}
The results are reported numerically in Table~\ref{tab:percentage_above_baseline_mean} and graphically on Figure~\ref{fig:cont_large} for $|\mathcal{D}| = 10,000$ and Figure~\ref{fig:cont_small} for $|\mathcal{D}| = 3,000$.

\textit{Empiric baseline polices.}
On Figure~\ref{fig:cont_large}, we observe that the baseline policies $\widehat{\pi}_{c}$ and $\widehat{\pi}_{n}$ have a performance poorer than the true behavior policy $\pi_b$.
On the one hand, the neural-based baseline estimate $\widehat{\pi}_{n}$ can get values close to the performance of the true behavior policy, however, it has a high variance and even the $90\%$-quantile is below the mean of the true policy.
On the other hand, the count-based policy $\widehat{\pi}_{c}$ has a low variance, but it has a much lower mean performance. In general, we observe a larger performance loss than in finite MDPs between the true baseline and the estimated baseline.

\textit{SPIBB.} With SPIBB, the neural-based baseline estimate leads to better results for all indicators. The loss in average performance makes it worse than RaMDP in the $|\mathcal{D}| = 10,000$ datasets, but it is more reliable and yields more consistently to policy improvements. On the $|\mathcal{D}| = 3,000$ datasets, it demonstrates a higher robustness with respect to the small datasets, still compared to RaMDP.

\textit{Soft-SPIBB.} The Soft-SPIBB results with baseline estimates are impressive. The loss of performance with respect to Soft-SPIBB with the true baseline is minor. We highlight that, although the policy based on pseudo-counts has a lower performance than the true one (1 point difference), it still achieves a strong performance when used with Soft-SPIBB (less than 0.1 point difference). This indicates that the proposed method is robust with respect to the performance of the estimated policy. It seems that Soft-SPIBB changes are much more forgiving the baseline approximations.

\textit{Small dataset.}
The experiment with a small dataset  $|\mathcal{D}|  = 3,000$ (Figure~\ref{fig:cont_small}) aims to evaluate the robustness of these algorithms.
We observe that the estimated policies have a performance even lower than in the experiment with $|\mathcal{D}|  = 10,000$.
While RaMDP's performance indicators dramatically plummet, even largely lower than the behavioural cloning policies,
the algorithm SPIBB using the estimated policies usually returns policies with a performance similar to the true baseline $\pi_b$.
Most exciting, the algorithm Soft-SPIBB manages to improve upon $\pi_b$ with all the baselines policies, obtaining a mean performance above the average performance of $\pi_b$, and a 10\%-quantile slightly lower than that of the true baseline when using the estimated policies.

\textit{Hyper-parameter sensitivity.}
The hyper-parameter search
gave us extra insights on the behavior of the algorithms SPIBB and Soft-SPIBB using estimated baselines.
We noticed that these algorithms do not have a high sensitivity to their hyper-parameters, since the performance is stable in a wide range of values, specially the Soft-SPIBB variations. We sometimes notice a tradeoff that has to be made between variance reduction and expectation maximization.

\section{Conclusion}
\label{sec:conclusions}
This paper addresses the problem of performing safe policy improvement in batch RL without direct access to the baseline, \textit{i.e.} the behavioural policy of the dataset. We provide the first theoretical guarantees for safe policy improvement in this setting, and show on finite and continuous MDPs that the algorithm is tractable and significantly outperforms all competing algorithms that do not have access to the baseline. We also empirically confirm the limits of the approach when the number of trajectories in the dataset is low.

Currently, the limitation of SPIBB methods is the lack of algorithms to compute the parametric uncertainty of the estimated model. \cite{Bellemare2016,Fox2018,Burda2019} investigated some methods for optimism-based exploration, which proved to not be robust enough for pessimism based purpose, where there is a requirement for exhaustiveness. Our future work in priority addresses this issue, but also the multi-batch setting, when there are several sequential updates~\cite{Laroche2019MultiBatch}, extending the method to continuous action spaces~\cite{Kumar2019}, and investigating the use of SPIBB in a full online setting, as a value estimation stabilizer.

\bibliographystyle{ACM-Reference-Format}
\bibliography{biblio}

\end{document}